\documentclass{article}

\PassOptionsToPackage{numbers}{natbib}
\usepackage[final]{neurips_2019}

\usepackage[utf8]{inputenc} %
\usepackage[T1]{fontenc}    %
\usepackage{hyperref}       %
\usepackage{url}            %
\usepackage{booktabs}       %
\usepackage{amsfonts}       %
\usepackage{nicefrac}       %
\usepackage{microtype}      %

\usepackage{amsmath,amsfonts,bm,amssymb,amsthm}

\newtheorem{lemma}{Lemma}

\def\eqref#1{equation~\ref{#1}}

\def\1{\bm{1}}

\def\va{{\bm{a}}}

\def\vs{{\bm{s}}}

\DeclareMathAlphabet{\mathsfit}{\encodingdefault}{\sfdefault}{m}{sl}
\SetMathAlphabet{\mathsfit}{bold}{\encodingdefault}{\sfdefault}{bx}{n}

\def\gA{{\mathcal{A}}}

\def\gD{{\mathcal{D}}}

\def\gL{{\mathcal{L}}}
\def\gM{{\mathcal{M}}}

\def\gP{{\mathcal{P}}}

\def\gS{{\mathcal{S}}}

\newcommand{\KL}{D_{\mathrm{KL}}}

\DeclareMathOperator*{\argmax}{arg\,max}
\DeclareMathOperator*{\argmin}{arg\,min}

\usepackage{graphicx}
\usepackage{algorithmic}
\usepackage{algorithm}
\usepackage{xcolor}
\usepackage{arydshln}
\usepackage{enumitem}
\usepackage{multirow}

\makeatletter
\newcommand{\printfnsymbol}[1]{%
  \textsuperscript{\@fnsymbol{#1}}%
}
\newcommand{\prg}[1]{\vspace{0.05in}\noindent\textbf{{#1}}}

\title{Meta-Inverse Reinforcement Learning with Probabilistic Context Variables}

\author{
   Lantao Yu\thanks{Equal contribution.}~~, Tianhe Yu\printfnsymbol{1}, Chelsea Finn, Stefano Ermon\\
   Department of Computer Science, Stanford University\\
  Stanford, CA 94305 \\
  \texttt{\{lantaoyu,tianheyu,cbfinn,ermon\}@cs.stanford.edu} \\
}

\begin{document}

\maketitle

\begin{abstract}
Providing a suitable reward function to reinforcement learning can be difficult in many real world applications. 
While inverse reinforcement learning (IRL) holds promise for automatically learning reward functions from demonstrations, several major challenges remain. First, existing IRL methods learn reward functions from scratch, requiring large numbers of demonstrations to correctly infer the reward for each task the agent may need to perform. Second,
existing methods typically assume homogeneous demonstrations for a single behavior or task, while in practice, it 
might be easier to collect 
datasets of heterogeneous but related behaviors. To this end, we propose a deep latent variable model that is capable of learning rewards from 
demonstrations of distinct but related tasks in an unsupervised way.
Critically, 
our model can 
infer 
rewards for new, structurally-similar tasks from a single demonstration. Our experiments on multiple continuous control tasks demonstrate the effectiveness of our approach compared to state-of-the-art imitation and inverse reinforcement learning methods.
\end{abstract}

%
\section{Introduction}
%
While reinforcement learning (RL) has been 
successfully applied to a range of
decision-making and control tasks in the real world, it relies on a key assumption: having access to a  well-defined reward function that measures progress towards the completion of the task. Although it can be straightforward to provide a high-level description of success conditions for a task, existing RL algorithms usually require a more informative signal to expedite exploration and learn complex behaviors in a reasonable time. While reward functions can be hand-specified, reward engineering 
can require significant human effort. 
Moreover, for many real-world tasks, it can be challenging to manually design reward functions that actually benefit RL training, and reward mis-specification can hamper autonomous learning~\cite{amodei2016}. 

Learning from demonstrations~\cite{imitation_survey} sidesteps the reward specification problem by instead learning directly from expert demonstrations, which can be obtained through teleoperation~\cite{vr_imitation} or from humans experts~\cite{yu2018daml}. Demonstrations can often be easier to provide than rewards, as humans can complete many real-world tasks quite efficiently. 
Two major methodologies of learning from demonstrations include imitation learning and inverse reinforcement learning.
Imitation learning is simple and often exhibits good performance~\cite{vr_imitation, gail}. However, it lacks the ability to transfer learned policies to new settings where the task specification remains the same but the underlying environment dynamics change.
As the reward function is often considered as the most succinct, robust and transferable representation of a task~\citep{abbeel2004apprenticeship,fu2017learning},
the problem of inferring reward functions from expert demonstrations, \emph{i.e.} inverse RL (IRL)~\cite{ng2000irl}, is important to consider.

While appealing, IRL still typically relies on large amounts of high-quality expert data, and it can be prohibitively expensive to collect demonstrations that cover all kinds of variations in the wild (\emph{e.g.} opening all kinds of doors or navigating to all possible target positions). As a result, these methods are data-inefficient, particularly when learning rewards for individual tasks in isolation, starting from scratch. %
On the other hand, meta-learning~\cite{schmidhuber1987evolutionary, bengio1991}, also known as learning to learn, seeks to exploit the structural similarity among a distribution of tasks and optimizes for rapid adaptation to unknown settings with a limited amount of data. As the reward function is able to succinctly capture the structure of a reinforcement learning task, \emph{e.g.} the goal to achieve, 
it is promising to develop methods that can quickly infer the structure of a new task, \emph{i.e.} its reward, %
and train a policy to adapt to it. \citet{xu2018learning} and \citet{multitaskirl} have proposed approaches that combine IRL and gradient-based meta-learning~\cite{finn2017maml}, which provide promising results on deriving generalizable reward functions. However, they have been limited to tabular MDPs~\cite{xu2018learning} or settings with provided task distributions~\cite{multitaskirl}, which are challenging to gather in real-world applications.

The primary contribution of this paper is a new framework, termed Probabilistic Embeddings for Meta-Inverse Reinforcement Learning (PEMIRL), which enables meta-learning of rewards from \emph{unstructured} multi-task demonstrations. In particular, PEMIRL combines and integrates ideas from context-based meta-learning~\cite{rl2,pearl}, deep latent variable generative models~\cite{vae}, and maximum entropy inverse RL~\citep{ziebart2008maximum,ziebart2010modeling}, into a unified graphical model (see Figure~\ref{fig:graphical_model} in Appendix~\ref{app:graphical_model}) that bridges the gap between few-shot reward inference and learning from unstructured, heterogeneous demonstrations.
PEMIRL can learn robust reward functions that generalize to new tasks with a \emph{single} demonstration on complex domains with continuous state-action spaces, while meta-training on a set of unstructured demonstrations without specified task groupings or labeling for each demonstration. Our experiment results on various continuous control tasks including Point-Maze, Ant, Sweeper, and Sawyer Pusher demonstrate the effectiveness and scalability of our method.
\section{Preliminaries}\label{sec:preliminary}
%
\prg{Markov Decision Process (MDP).}   A discrete-time finite-horizon MDP is defined by a tuple $(T, \gS, \gA, P, r, \eta)$, where $T$ is the time horizon; $\gS$ is the state space; $\gA$ is the action space; $P: \gS \times \gA \times \gS \to [0,1]$ describes the (stochastic) transition process between states; $r: \gS \times \gA \to \mathbb{R}$ is a bounded reward function; $\eta \in \mathcal{P}(\gS)$ specifies the initial state distribution, where $\gP(\gS)$ denotes the set of probability distributions over the state space $\gS$. We use $\tau$ to denote a trajectory, \emph{i.e.} a sequence of state action pairs for one episode. We also use $\rho_\pi(s_t)$ and $\rho_\pi(s_t, a_t)$ to denote the state and state-action marginal distribution encountered when executing a policy $\pi(a_t|s_t)$.

\prg{Maximum Entropy Inverse Reinforcement Learning (MaxEnt IRL)}.  
The maximum entropy reinforcement learning (MaxEnt RL) objective is defined as:
\begin{equation}
    \max_\pi \sum_{t=1}^T \mathbb{E}_{(s_t, a_t) \sim \rho_\pi}[r(s_t, a_t) + \alpha \mathcal{H}(\pi(\cdot|s_t))]
\end{equation}
which augments the reward function with a causal entropy regularization term $\mathcal{H}(\pi) = \mathbb{E}_\pi [-\log \pi(a|s)]$. Here $\alpha$ is an optional parameter to control the relative importance of reward and entropy. For notational simplicity, without loss of generality, in the following we will assume $\alpha=1$.
Given some expert policy $\pi_E$ that is obtained by above MaxEnt RL procedure
, the MaxEnt IRL framework \citep{ziebart2008maximum} aims to find a reward function that rationalizes the expert behaviors, which can be interpreted as solving the following maximum likelihood estimation (MLE) problem:
\begin{gather}
     p_\theta(\tau) \propto \left[\eta(s_1) \prod_{t=1}^T P(s_{t+1}|s_t,a_t)\right] \exp\left( \sum_{t=1}^T r_\theta (s_t, a_t) \right) = \overline{p_\theta}(\tau) \label{eq:irl-mle} \\
     \argmin_\theta \KL(p_{\pi_E(\tau)}||p_\theta(\tau)) = \argmax_\theta \mathbb{E}_{p_{\pi_E}(\tau)}\left[\log p_\theta(\tau)\right] = \mathbb{E}_{\tau \sim \pi_E}\left[ \sum_{t=1}^{T} r_\theta(s_t,a_t) \right] - \log Z_\theta \nonumber
\end{gather}
Here, $\theta$ is the parameter of the reward function and $Z_\theta$ is the partition function, \emph{i.e.} $\int \overline{p_\theta}(\tau) d \tau$, an integral over all possible trajectories consistent with the environment dynamics. $Z_\theta$ is intractable to compute when state-action spaces are large or continuous, or environment dynamics are unknown. 
\citet{finn2016connection} and \citet{fu2017learning} proposed the adversarial IRL (AIRL) framework as an efficient sampling-based approximation to MaxEnt IRL, which resembles Generative Adversarial Networks \citep{goodfellow2014generative}. Specially, in AIRL, there is a discriminator $D_\theta$ (a binary classifier) parametrized by $\theta$ and an adaptive sampler $\pi_\omega$ (a policy) parametrized by $\omega$. The discriminator takes a particular form:
$
D_\theta(s,a) = \exp(f_\theta(s,a))/(\exp(f_\theta(s,a)) + \pi_\omega(a|s))\label{eq:airl-discrim}
$
, where $f_\theta(s,a)$ is the learned reward function and $\pi_\omega(a|s)$ is pre-computed as
 an input to the discriminator.
The discriminator is trained to distinguish between the trajectories sampled from the expert and the adaptive sampler; while the adaptive sampler $\pi_\omega(a|s)$ is trained to maximize $\mathbb{E}_{\rho_{\pi_\omega}}[\log D_\theta(s,a) - \log(1-D_\theta(s,a))]$, 
which is equivalent to maximizing the following entropy regularized policy objective (with $f_\theta(s, a)$ serving as the reward function):
\begin{align}
\mathbb{E}_{\pi_\omega} \left[\sum_{t=1}^T \log(D_\theta(s_t, a_t)) - \log(1 - D_\theta(s_t, a_t))\right] = 
\mathbb{E}_{\pi_\omega} \left[ \sum_{t=1}^T f_\theta(s_t, a_t) - \log \pi_\omega(a_t|s_t) \right]
\label{eq:airl-policy}
\end{align}
Under certain conditions, it can be shown that the learned reward function will recover the ground-truth reward up to a constant (Theorem C.1 in \cite{fu2017learning}).

\section{Probabilistic Embeddings for Meta-Inverse Reinforcement Learning}
%
\subsection{Problem Statement}

Before defining our meta-inverse reinforcement learning problem (Meta-IRL), we first define the concept of optimal context-conditional policy.

We start by generalizing the notion of MDP with a probabilistic context variable denoted as $m \in \mathcal{M}$, where $\gM$ is the (discrete or continuous) value space of $m$. For example, in a navigation task, the context variables could represent different goal positions in the environment.
Now, each component of the MDP has an additional dependency on the 
context variable $m$. For example, by slightly overloading the notation, the reward function is now defined as $r: \gS \times \gA \times \gM \to \mathbb{R}$.
For simplicity, the state space, action space, initial state distribution and transition dynamics are often assumed to be independent of $m$ \citep{rl2,finn2017maml}, which we will follow in this work. Intuitively, different $m$'s correspond to different tasks with shared structures.

Given above definitions, the context-conditional trajectory distribution induced by a context-conditional policy $\pi: \gS \times \gM \to \gP(\gA)$ can be written as:
\begin{align}
    p_\pi(\tau=\{\vs_{1:T}, \va_{1:T}\}|m)
    = \eta(s_1) \prod_{t=1}^T \pi(a_t|s_t,m) P(s_{t+1}|s_t, a_t)
\end{align}

Let $p(m)$ denote the prior distribution of the latent context variable (which is a part of the problem definition). With the conditional distribution defined above, the optimal entropy-regularized context-conditional policy is defined as:
\begin{equation}
    \pi^* = \argmax_\pi \mathbb{E}_{m \sim p(m),~(\vs_{1:T}, \va_{1:T})\sim p_\pi(\cdot|m)}\left[\sum_{t=1}^T r(s_t, a_t, m) - \log \pi(a_t|s_t,m)\right]\label{eq:expert}
\end{equation}

Now, let us introduce the problem of Meta-IRL from heterogeneous multi-task demonstration data. Suppose there is some ground-truth reward function $r(s,a,m)$ and a corresponding expert policy $\pi_E(a_{t}|s_t, m)$ obtained by solving the optimization problem defined in Equation~(\ref{eq:expert}). %
Given a set of demonstrations \emph{i.i.d.} sampled from the induced marginal distribution $p_{\pi_E}(\tau) = \int_\gM p(m) p_{\pi_E}(\tau|m) dm$, the goal is to meta-learn an inference model $q(m|\tau)$ and a reward function $f(s,a,m)$, such that given some new demonstration $\tau_E$ generated by sampling $m' \sim p(m), \tau_E\sim p_{\pi_E}(\tau|m')$, with $\hat{m}$ being inferred as $\hat{m} \sim q(m|\tau_E)$, 
the learned reward function $f(s,a,\hat{m})$ and the ground-truth reward $r(s,a,m')$ will induce the same set of optimal policies~\cite{ng1999policy}.

Critically, we assume no knowledge of the prior task distribution $p(m)$, the latent context variable $m$ associated with each demonstration, nor the transition dynamics $P(s_{t+1}|s_t, a_t)$ during meta-training.
Note that the entire supervision comes from the provided unstructured demonstrations, which means we also do not assume further interactions with the experts as in \citet{ross2011reduction}.
\subsection{Meta-IRL with Mutual Information Regularization over Context Variables}
%

Under the framework of MaxEnt IRL, we first parametrize the context variable inference model $q_\psi(m|\tau)$ and the reward function $f_\theta(s,a,m)$ (where the input $m$ is inferred by $q_\psi$), 
The induced $\theta$-parametrized trajectory distribution is given by:
\begin{align}
    p_\theta(\tau=\{\vs_{1:T}, \va_{1:T}\}|m) = \frac{1}{Z(\theta)} \left[\eta(s_1) \prod_{t=1}^T P(s_{t+1}|s_t,a_t)\right] \exp\left(\sum_{t=1}^T f_\theta(s_t,a_t,m)\right)\label{eq:conditional}
\end{align}
where $Z(\theta)$ is the partition function, $i.e.$, an integral over all possible trajectories. Without further constraints over $m$, directly applying AIRL to learning the reward function (by augmenting each component of AIRL with an additional context variable $m$ inferred by $q_\psi$) could simply ignore $m$,
which is similar to the case of InfoGAIL \citep{li2017infogail}. 
Therefore, some connection between the reward function and the latent context variable $m$ need to be established.
With MaxEnt IRL, a parametrized reward function will induce a trajectory distribution. From the perspective of information theory, the mutual information between the context variable $m$
and the trajectories sampled from the reward induced distribution will provide an ideal measure for such a connection.

Formally, the mutual information between two random variables $m$ and $\tau$ under joint distribution $p_\theta(m, \tau) = p(m) p_\theta(\tau|m)$ is given by:
\begin{align}
    I_{p_\theta}(m;\tau) = \mathbb{E}_{m \sim p(m), \tau \sim p_\theta(\tau|m)} [\log p_\theta(m|\tau) - \log p(m)] \label{eq:mutual}
\end{align}
where 
$p_\theta(\tau|m)$ is the conditional distribution (Equation~(\ref{eq:conditional})), and $p_\theta(m|\tau)$ is the corresponding posterior distribution. 

As we do not have access to the prior distribution $p(m)$ and posterior distribution $p_\theta(m|\tau)$, directly optimizing the mutual information in Equation~(\ref{eq:mutual}) is intractable. 
Fortunately, we can leverage $q_\psi(m|\tau)$ as a variational approximation to $p_\theta(m|\tau)$ to reason about the uncertainty over tasks, as well as conduct approximate sampling from $p(m)$ (we will elaborate this later in Section~\ref{sec:tractability}). Formally, let $p_{\pi_E}(\tau)$ denote the expert trajectory distribution, we have the following desiderata:

\textbf{Desideratum 1}. Matching conditional distributions: $\mathbb{E}_{p(m)}\left[ \KL(p_{\pi_E}(\tau|m)||p_\theta(\tau|m))\right] = 0$

\textbf{Desideratum 2}. Matching posterior distributions: $\mathbb{E}_{p_\theta(\tau)} [\KL(p_\theta(m|\tau)||q_\psi(m|\tau))] = 0$

The first desideratum will encourage the $\theta$-induced conditional trajectory distribution to match the empirical distribution implicitly defined by the expert demonstrations, which is equivalent to the MLE objective in the MaxEnt IRL framework. Note that they also share the same marginal distribution over the context variable $p(m)$, which implies that matching the conditionals in Desideratum 1 will also encourage the joint distributions, conditional distributions $p_{\pi_E}(m|\tau)$ and $p_\theta(m|\tau)$, and marginal distributions over $\tau$ to be matched.
The second desideratum will encourage the variational posterior $q_\psi(m|\tau)$ to be a good approximation to $p_\theta(m|\tau)$ such that $q_\psi(m|\tau)$ can correctly infer the latent context variable given a new expert demonstration sampled from a new task.

With the mutual information (Equation~(\ref{eq:mutual})) being the objective, and Desideratum 1 and 2 being the constraints, the meta-inverse reinforcement learning with probabilistic context variables problem can be interpreted as a constrained optimization problem, whose Lagrangian dual function is given by:
\begin{align}
    \min_{\theta, \psi} -I_{p_\theta}(m;\tau) + \alpha \cdot
    \mathbb{E}_{p(m)}\left[ \KL(p_{\pi_E}(\tau|m)||p_\theta(\tau|m))\right]
     + \beta \cdot \mathbb{E}_{p_\theta(\tau)} [\KL(p_\theta(m|\tau)||q_\psi(m|\tau))]
\end{align}
With the Lagrangian multipliers taking specific values ($\alpha=1, \beta=1$) \citep{zhao2018information}, the above Lagrangian dual function can be rewritten as:
{
\begin{align}
    &\min_{\theta, \psi}~ \mathbb{E}_{p(m)}\left[ \KL(p_{\pi_E}(\tau|m)||p_\theta(\tau|m))\right]  + \mathbb{E}_{p_\theta(m, \tau)} \left[\log \frac{p(m)}{p_\theta(m|\tau)} + \log \frac{p_\theta(m|\tau)}{q_\psi(m|\tau)} \right] \nonumber\\
    \equiv&\max_{\theta, \psi} -\mathbb{E}_{p(m)}\left[ \KL(p_{\pi_E}(\tau|m)||p_\theta(\tau|m))\right] + ~\mathbb{E}_{m\sim p(m), \tau \sim p_\theta(\tau|m)} [\log q_\psi(m|\tau)] \label{eq:equivalence} \\
    = &\max_{\theta, \psi}-\mathbb{E}_{p(m)}\left[ \KL(p_{\pi_E}(\tau|m)||p_\theta(\tau|m))\right] +  \mathcal{L}_\text{info}(\theta, \psi)\label{eq:objective}
\end{align}
\small}%
\vspace{-0.4cm}

Here the negative entropy term $-H_p(m) = \mathbb{E}_{p_\theta(m, \tau)} [\log p(m)]= \mathbb{E}_{p(m)}[\log p(m)]$ is omitted (in Eq.~(\ref{eq:equivalence})) as it can be treated as a constant in the optimization procedure of parameters $\theta$ and $\psi$.

\subsection{Achieving Tractability with Sampling-Based Gradient Estimation}\label{sec:tractability}

Note that Equation~(\ref{eq:objective}) cannot be evaluated directly, as the first term requires estimating the KL divergence between the empirical expert distribution and the energy-based trajectory distribution $p_\theta(\tau|m)$ (induced by the $\theta$-parametrized reward function), 
and the second term requires sampling from it. 
For the purpose of optimizing the first term in Equation~(\ref{eq:objective}), as introduced in Section~\ref{sec:preliminary}, we can employ the adversarial reward learning framework \citep{fu2017learning} to construct an efficient sampling-based approximation to the maximum likelihood objective. Note that different from the original AIRL framework, now the adaptive sampler $\pi_\omega(a|s, m)$ is additionally conditioned on the context variable $m$. Furthermore, we here introduce the following lemma, which will be helpful for deriving the optimization of the second term in Equation~(\ref{eq:objective}).

\begin{lemma}\label{lemma:airl}
In context variable augmented Adversarial IRL (with the adaptive sampler being $\pi_\omega(a|s,m)$ and the discriminator being $D_\theta(s,a,m) = \frac{\exp(f_\theta(s,a,m))}{\exp(f_\theta(s,a,m)) + \pi_\omega(a|s,m)}$)
, under deterministic dynamics, when training the adaptive sampler $\pi_\omega$ with reward signal $(\log D_\theta - \log(1-D_\theta))$ to optimality, the trajectory distribution induced by $\pi_\omega^*$ corresponds to the maximum entropy trajectory distribution with $f_\theta(s,a,m)$ serving as the reward function:
$$p_{\pi_{\omega}^*}(\tau|m) = \frac{1}{Z_\theta} \left[\eta(s_1) \prod_{t=1}^T P(s_{t+1}|s_t,a_t)\right] \exp\left(\sum_{t=1}^T f_\theta(s_t,a_t,m)\right) = p_\theta(\tau|m)$$
\end{lemma}
\begin{proof} 
See Appendix~\ref{appendix:airl}.
\end{proof}

Now we are ready to introduce how to approximately optimize the second term of the objective in Equation~(\ref{eq:objective}) w.r.t. $\theta$ and $\psi$. First, we observe that the gradient of $\gL_\text{info}(\theta, \psi)$ w.r.t. $\psi$ is given by:
\begin{align}
    \frac{\partial}{\partial \psi} \gL_\text{info}(\theta, \psi) = \mathbb{E}_{m \sim p(m), \tau \sim p_\theta(\tau|m)} \frac{1}{q(m|\tau, \psi)} \frac{\partial q(m|\tau, \psi)}{\partial \psi}\label{eq:gradient-psi}
\end{align}
Thus to construct an estimate of the gradient in Equation~(\ref{eq:gradient-psi}), we need to obtain samples from the $\theta$-induced trajectory distribution $p_\theta(\tau|m)$.
With Lemma~\ref{lemma:airl}, we know that when the adaptive sampler $\pi_\omega$ in AIRL is trained to optimality, we can use $\pi_\omega^*$ to construct samples, as the trajectory distribution $p_{\pi_\omega^*} (\tau|m)$ matches the desired distribution $p_\theta(\tau|m)$. 

Also note that the expectation in Equation~(\ref{eq:gradient-psi}) is also taken over the prior task distribution $p(m)$. In cases where we have access to the ground-truth prior distribution, we can directly sample $m$ from it and use $p_{\pi_\omega^*}(\tau|m)$ to construct a gradient estimation. For the most general case, where we do not have access to $p(m)$ but instead have expert demonstrations sampled from $p_{\pi_E}(\tau)$, we use the following generative process: 
\begin{equation}
\tau \sim p_{\pi_E(\tau)}, m \sim q_\psi(m|\tau)\label{eq:generate-m}
\end{equation}
to synthesize latent context variables, which approximates the prior task distribution when $\theta$ and $\psi$ are trained to optimality.

To optimize $\gL_\text{info}(\theta, \psi)$ w.r.t. $\theta$, which is an important step of updating the reward function parameters such that it encodes the information of the latent context variable, different from the optimization of Equation~(\ref{eq:gradient-psi}), we cannot directly replace $p_\theta(\tau|m)$ with $p_{\pi_\omega}(\tau|m)$. The reason is that we can only use the approximation of $p_\theta$ to do inference (\emph{i.e.} computing the value of an expectation). When we want to optimize an expectation ($\gL_\text{info}(\theta, \psi)$) w.r.t. $\theta$ and the expectation is taken over $p_\theta$ itself, we cannot instead replace $p_\theta$ with $\pi_\omega$ to do the sampling for estimating the expectation. In the following, we discuss how to estimate the gradient of $\gL_\text{info}(\theta, \psi)$ w.r.t. $\theta$ with empirical samples from $\pi_\omega$.

\begin{lemma}\label{lemma:gradient}
The gradient of $\gL_\text{info}(\theta, \psi)$ w.r.t. $\theta$ can be estimated with:
\begin{align}
    \mathbb{E}_{m \sim p(m), \tau \sim p_{\pi_\omega^*}(\tau|m)}\left[ \log q_\psi(m|\tau) \left[ \sum_{t=1}^T \frac{\partial}{\partial \theta} f_\theta(s_t,a_t,m) - \mathbb{E}_{\tau' \sim p_{\pi_\omega^*}(\tau|m)}\sum_{t=1}^T \frac{\partial}{\partial \theta} f_\theta(s'_t,a'_t,m)\right] \right] \nonumber
\end{align}
When $\omega$ is trained to optimality, the estimation is unbiased.
\end{lemma}

\begin{proof}
See Appendix~\ref{appendix:gradient}.
\end{proof}
With Lemma~\ref{lemma:gradient}, as before, we can use the generative process in Equation~(\ref{eq:generate-m}) to sample $m$ and use the conditional trajectory distribution $p_{\pi_\omega^*}(\tau|m)$ to sample trajectories for estimating $\frac{\partial}{\partial\theta}\gL_\text{info}(\theta, \psi)$.
The overall training objective of PEMIRL is:
\begin{align}
    & \min_\omega \max_{\theta, \psi}~\mathbb{E}_{\tau_E \sim p_{\pi_E}(\tau), m \sim q_\psi(m|\tau_E), (s,a) \sim \rho_{\pi_\omega}(s,a|m)} \log (1-D_\theta(s,a,m)) + \nonumber\\
    &~~~~~~~~~~~~~~~~~\mathbb{E}_{\tau_E \sim p_{\pi_E}(\tau), m \sim q_\psi(m|\tau_E)} \log(D_\theta(s,a,m)) + \gL_\text{info}(\theta, \psi)\\
    & \text{where}~D_\theta(s, a, m) = \exp(f_\theta(s, a, m))/(\exp(f_\theta(s, a, m)) + \pi_\omega(a|s, m))\nonumber
\end{align}
We summarize the meta-training procedure in Algorithm~\ref{meta_algorithm} and the meta-test procedure in Appendix~\ref{appendix:meta-test}. 

\begin{algorithm}[h]
   \caption{PEMIRL Meta-Training}
   \footnotesize
   \label{meta_algorithm}
\begin{algorithmic}
   \STATE {\bfseries Input:} Expert trajectories $\mathcal{D}_E = \{\tau_E^j\}$; Initial parameters of $f_\theta, \pi_\omega, q_\psi$.
   \REPEAT
  \STATE Sample two batches of unlabeled demonstrations: $\tau_E, \tau'_E \sim \mathcal{D}_E$
   \STATE Infer a batch of latent context variables from the sampled demonstrations: $m \sim q_\psi(m|\tau_E)$
   \STATE Sample trajectories $\gD$ from $\pi_\omega(\tau|m)$, with the latent context variable fixed during each rollout and included in $\gD$.
   \STATE Update $\psi$ to increase $\gL_\text{info}(\theta, \psi)$ with gradients in Equation~(\ref{eq:gradient-psi}), with samples from $\gD$.
   \STATE Update $\theta$ to increase $\gL_\text{info}(\theta, \psi)$ with gradients in Equation~(\ref{eq:gradient-theta}), with samples from $\gD$.
   \STATE Update $\theta$ to decrease the binary classification loss:
   \STATE \begin{center}
       $\mathbb{E}_{(s,a,m) \sim \gD}[\nabla_\theta \log D_\theta(s,a,m)] + \mathbb{E}_{\tau'_E \sim \gD_E, m \sim q_\psi(m|\tau'_E)}[\nabla_\theta \log(1 - D_\theta(s,a,m))]$
   \end{center}
   \STATE Update $\omega$ with TRPO to increase the following objective: $\mathbb{E}_{(s,a,m) \sim \gD} [\log D_\theta(s,a,m)]$
   \UNTIL{Convergence}
   \STATE {\bfseries Output:} Learned inference model $q_\psi(m|\tau)$, reward function $f_\theta(s,a,m)$ and policy $\pi_\omega(a|s, m)$.
\end{algorithmic}
\end{algorithm}

%
\section{Related Work}
%
%

Inverse reinforcement learning (IRL), first introduced by~\citet{ng2000irl}, is the problem of learning reward functions directly from expert demonstrations. Prior work tackling IRL include margin-based methods~\cite{abbeel2004apprenticeship, ratliff2006} and maximum entropy (MaxEnt) methods~\cite{ziebart2008maximum}. %
Margin-based methods suffer from being an underdefined problem, while MaxEnt requires the algorithm to solve the forward RL problem in the inner loop, making it challenging to use in non-tabular settings.
Recent works have scaled MaxEnt IRL to large function approximators, such as neural networks, by only partially solving the forward problem in the inner loop, developing an adversarial framework for IRL~\cite{finn2016connection, finn2016guided, fu2017learning, peng2018}. Other imitation learning approaches~\cite{gail, li2017infogail, hausman2017multi, Kuefler2018} are also based on the adversarial framework, but they do not recover a reward function.
We build upon the ideas in these single-task IRL works.
Instead of considering the problem of learning reward functions for a single task, we aim at the problem of inferring a reward that is disentangled from the environment dynamics and can quickly adapt to new tasks from a single demonstration by leveraging prior data.

We base our work on the problem of meta-learning. Prior work has proposed memory-based methods~\cite{rl2, mann, metanetworks, pearl} and methods that learn an optimizer and/or a parameter initialization~\cite{Andrychowicz2016, li17optimize, finn2017maml}. We adopt a memory-based meta-learning method similar to~\cite{pearl}, which uses a deep latent variable generative model~\cite{vae} to infer different tasks from demonstrations. While prior multi-task and meta-RL methods~\cite{hausman2018learning, pearl, saemundsson2018meta} have investigated the effectiveness of applying latent variable generative models to learning task embeddings, we focus on the IRL problem instead. Meta-IRL~\cite{xu2018learning, multitaskirl} incorporates meta-learning and IRL, showing fast adaptation of the reward functions to unseen tasks. Unlike these approaches, our method is not restrictred to discrete tabular settings and does not require access to grouped demonstrations sampled from a task distribution.
Meanwhile, one-shot imitation learning~\cite{duan2017, finn2017one, yu2018daml, yu2018hil} demonstrates impressive results on learning new tasks using a single demonstration; yet, they also require paired demonstrations from each task and hence need prior knowledge on the task distribution. More importantly, one-shot imitation learning approaches only recover a policy, and cannot use additional trials to continue to improve, which is possible when a reward function is inferred instead.
Several prior approaches for multi-task imitation learning~\cite{li2017infogail, hausman2017multi, sharma2018directedinfogail} propose to use unstructured demonstrations without knowing the task distribution, but they neither study quick generalization to new tasks nor provide a reward function. Our work is thus driven by the goal of extending meta-IRL to addressing challenging high-dimensional control tasks with the help of an unstructured demonstration dataset.

\section{Experiments}
%
\label{sec:experiments}

In this section, we seek to investigate the following two questions: (1) Can PEMIRL learn a policy with competitive few-shot generalization abilities compared to one-shot imitation learning methods using only unstructured demonstrations? (2) Can PEMIRL efficiently infer robust reward functions of new continuous control tasks where one-shot imitation learning fails to generalize, enabling an agent to continue to improve with more trials?

We evaluate our method on four simulated domains using the Mujoco physics engine~\cite{mujoco}. To our knowledge, there's no prior work on designing meta-IRL or one-shot imitation learning methods for complex domains with high-dimensional continuous state-action spaces with unstructured demonstrations. Hence, we also designed the following variants of existing state-of-the-art (one-shot) imitation learning and IRL methods so that they can be used as fair comparisons to our method:
\begin{itemize}[leftmargin=.3in]
\item \textbf{AIRL}: The original AIRL algorithm without incorporating latent context variables, trained across all demonstrations.
\item \textbf{Meta-Imitation Learning with Latent Context Variables (Meta-IL)}: As in~\citep{pearl}, we use the inference model $q_\psi(m|\tau)$ to infer the context of a new task from a single demonstrated trajectory, denoted as $\hat{m}$, and then train the conditional imitaiton policy $\pi_\omega(a|s,\hat{m})$ using the same demonstration. This approach also resembles~\cite{duan2017}.
\item \textbf{Meta-InfoGAIL}: Similar to the method above, except that an additional discriminator $D(s,a)$ is introduced to distinguish between expert and sample trajectories, and trained along with the conditional policy using InfoGAIL~\cite{li2017infogail} objective.
\end{itemize}
We use trust region policy optimization (TRPO)~\cite{trpo} as our policy optimization algorithm across all methods. We collect demonstrations by training experts with TRPO using ground truth reward. However, the ground truth reward is not available to imitation learning and IRL algorithms. We provide full hyperparameters, architecture information, data efficiency, and experimental setup details in Appendix~\ref{appendix:experiment}. We also include ablation studies on sensitivity of the latent dimensions, importance of the mutual information objective and the performance on stochastic environments in Appendix~\ref{app:ablation}. Full video results are on the anonymous supplementary website\footnote{Video results can be found at: \url{https://sites.google.com/view/pemirl}} and our code is open-sourced on GitHub\footnote{Our implementation of PEMIRL can be found at: \url{https://github.com/ermongroup/MetaIRL}}.

\begin{figure*}[!t]
    \centering
    \includegraphics[width=0.9\columnwidth]{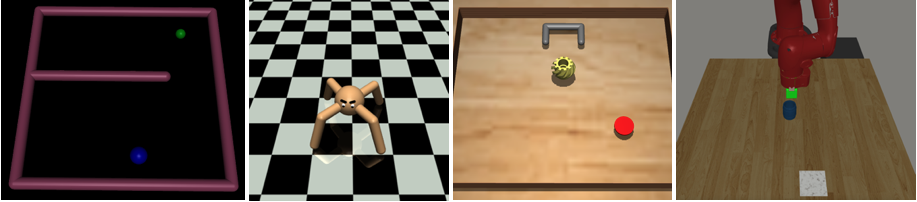}
    \caption{\textbf{Experimental domains} (left to right):
    Point-Maze, Ant, Sweeper, and Sawyer Pusher. 
    }
    \label{fig:setup}
\end{figure*}

\subsection{Policy Performance on Test Tasks}
%
\label{sec:imitation_results}

\begin{table}[h]
    \begin{center}
    \begin{small}
    \begin{tabular}{l|c|c|c|c}
    \toprule
        & Point Maze & Ant & Sweeper & Sawyer Pusher  \\
      \midrule
      Expert & $-5.21 \pm 0.93$ & $968.80 \pm 27.11$ & $-50.86 \pm 4.75$ & $-23.36 \pm 2.54$\\
      Random & $-51.39 \pm 10.31$ & $-55.65 \pm 18.39$ & $-259.71 \pm 11.24$ & $-106.88 \pm 18.06$\\
      \midrule
      AIRL~\cite{fu2017learning} & $-18.15 \pm 3.17$ & $127.61 \pm 27.34$ & $-152.78 \pm 7.39$ & $-51.56 \pm 8.57$\\
      Meta-IL & $\mathbf{-6.68} \pm 1.51$ & $218.53 \pm 26.48$ & $-89.02 \pm 7.06$ & $-28.13 \pm 4.93$\\
      Meta-InfoGAIL & $-7.66 \pm 1.85$ & $\mathbf{871.93} \pm 31.28$ & $-87.06 \pm 6.57$ & $-27.56 \pm 4.86$\\
      PEMIRL (ours) & $-7.37 \pm 1.02$ & $846.18 \pm 25.81$ & $\mathbf{-74.17} \pm 5.31$ & $\mathbf{-27.16} \pm 3.11$ \\
      \bottomrule
    \end{tabular}
    \end{small}
    \end{center}
    \vspace{0.2cm}
    \caption{One-shot policy generalization to test tasks on four experimental domains.
    Average return and standard deviations are reported over $5$ runs.}
    \label{tbl:imitation_results}
\end{table}

We first answer our first question by showing that our method is able to learn a policy that can adapt to test tasks from a single demonstration, on four continuous control domains: \textbf{Point Maze Navigation}: In this domain, a pointmass needs to navigate around a barrier to reach the goal. Different tasks correspond to different goal positions and the reward function measures the distance between the pointmass and the goal position; \textbf{Ant}: Similar to~\cite{finn2017maml}, this locomotion task requires fast adaptation to walking directions of the ant where the ant needs to learn to move backward or forward depending on the demonstration;
\textbf{Sweeper}: A robot arm needs to sweep an object to a particular goal position. Fast adaptation of this domain corresponds to different goal locations in the plane;
\textbf{Sawyer Pusher}: A simulated Sawyer robot is required to push a mug to a variety of goal positions and generalize to unseen goals.
We illustrate the set-up for these experimental domains in Figure~\ref{fig:setup}. 

We summarize the results in Table~\ref{tbl:imitation_results}. PEMIRL achieves comparable imitation performance compared to Meta-IL and Meta-InfoGAIL, while AIRL is incapable of handling multi-task scenarios without incorporating the latent context variables.

\subsection{Reward Adaptation to Challenging Situations}
%

%
\begin{figure*}[!t]
    \centering
    \includegraphics[width=0.2088\columnwidth]{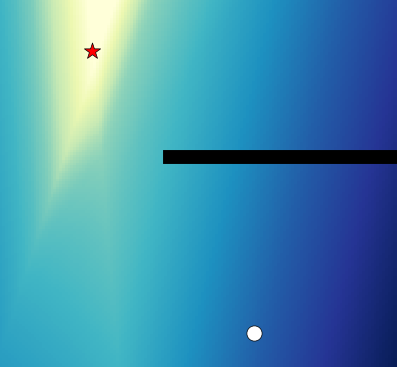}
    \includegraphics[width=0.2088\columnwidth]{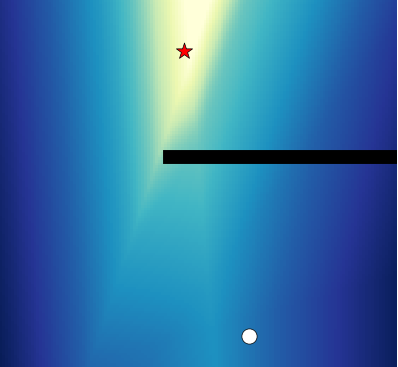}
    \includegraphics[width=0.2088\columnwidth]{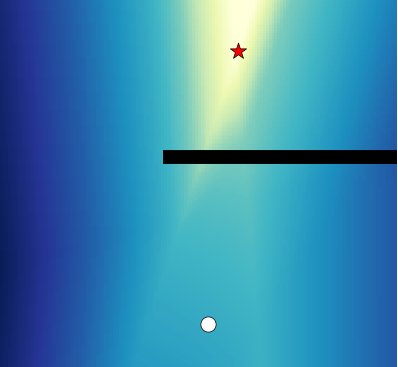}
    \includegraphics[width=0.252\columnwidth]{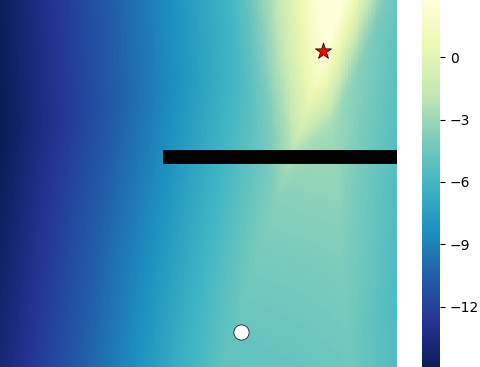}
    \caption{Visualizations of learned reward functions for point-maze navigation. The red star represents the target position and the white circle represents the initial position of the agent (both are different across different iterations). The black horizontal line represents the barrier that cannot be crossed. To show the generalization ability, the expert demonstration used to infer the target position are sampled from new target positions that have not been seen in the meta-training set.
    }
    \label{fig:maze_reward_fig}
\end{figure*}

\begin{figure*}[!t]
    \centering
    \includegraphics[width=0.9\columnwidth]{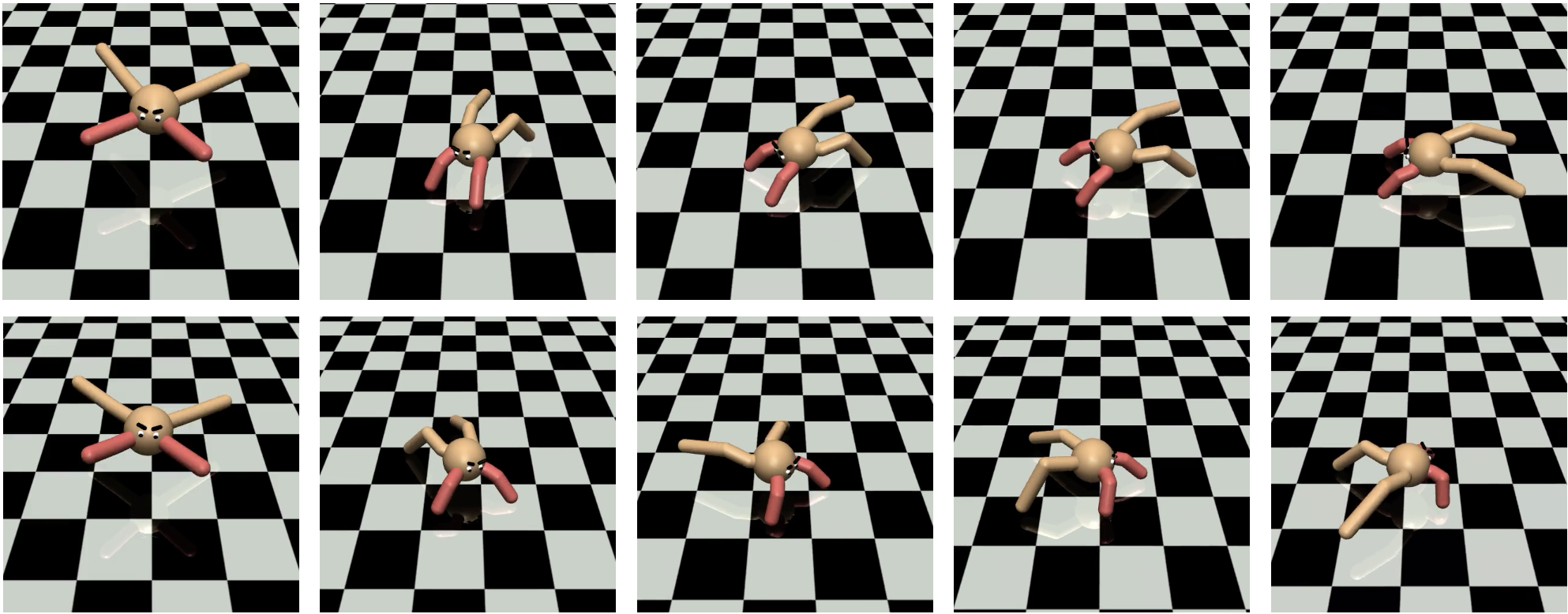}
    \caption{From top to bottom, we show the disabled ant running forward and backward respectively.
    }
    \label{fig:ant_film_strips}
\end{figure*}

After demonstrating that the policy learned by our method is able to achieve competitive ``one-shot'' generalization ability, we now answer the second question by showing PEMIRL learns a robust reward that can adapt to new and more challenging settings where the imitation learning methods and the original AIRL fail. Specifically, after providing the demonstration of an unseen task to the agent, we change the underlying environment dynamics but keep the same task goal. In order to succeed in the task with new dynamics, the agent must correctly infer the underlying goal of the task instead of simply mimicking the demonstration. We show the effectiveness of our reward generalization by training a new policy with TRPO using the learned reward functions on the new task.

\prg{Point-Maze Navigation with a Shifted Barrier}. Following the setup of ~\citet{fu2017learning}, at meta-test time, after showing a demonstration moving towards a new target position, we change the position of the barrier from left to right. 
As the agent must adapt by reaching the target with a different path from what was demonstrated during meta-training, it cannot succeed without correctly inferring the true goal (the target position in the maze) and learning from trial-and-error. As a result, all direct policy generalization approaches fail as all the policies are still directing the pointmass to the right side of the maze. As shown in Figure~\ref{fig:maze_reward_fig}, PEMIRL learns disentangled reward functions that successfully infer the underlying goal of the new task without much reward shaping. Such reward functions enable the RL agent to bypass the right barrier and reach the true goal position. The RL agent trained with the reward learned by AIRL also fail to bypass the barrier and
navigate to the target position, as without incorporating the latent context variables and treating the demonstration as multi-modal, AIRL learns an ``average'' reward and policy among different tasks. We also use the output of the discriminator of Meta-InfoGAIL as reward signals and evaluate its adaptation performance. The agent trained by this reward fails to complete the task since Meta-InfoGAIL does not explicitly optimize for reward learning and the discriminator output converges to uninformative uniform distribution at convergence.

\begin{table}[t]
    \begin{center}
    \setlength\aboverulesep{1.5pt}\setlength\belowrulesep{1.5pt}
    %
    \begin{tabular}{c|c|c|c}
    \toprule
      & Method & Point-Maze-Shift & Disabled-Ant  \\
      \midrule
      \multirow{3}{*}{\shortstack{
      Policy\\ Generalization}} & Meta-IL  & $-28.61 \pm 3.71$ & $-27.86 \pm 10.31$\\
      & Meta-InfoGAIL & $-29.40 \pm 3.05$ & $-51.08 \pm 4.81$\\
      & PEMIRL & $-28.93 \pm 3.59$ & $-46.77 \pm 5.54$ \\
      \midrule
      \multirow{3}{*}{\shortstack{Reward\\ Adaptation}} & AIRL & $-29.07 \pm 4.12$ & $-76.21 \pm 10.35$\\
      & Meta-InfoGAIL & $-29.72 \pm 3.11$ & $-38.73 \pm 6.41$ \\
      & PEMIRL (ours) & $\mathbf{-9.04} \pm 1.09$ & $\mathbf{152.62} \pm 11.75$\\
    \midrule
      & Expert & $-5.37 \pm 0.86$ & $331.17 \pm 17.82$\\
      \bottomrule
    \end{tabular}
    \end{center}
    \caption{Results on direct policy generalization and reward adaptation to challenging situations. Policy generalization examines if the policy learned by Meta-IL is able to generalize to new tasks with new dynamics, while reward adaptation tests if the learned RL can lead to efficient RL training in the same setting. The RL agent learned by PEMIRL rewards outperforms other methods in such challenging settings.
    }
    \label{tbl:reward_results}
    \vspace{-0.5cm}
\end{table}

\prg{Disabled Ant Walking}. As in ~\citet{fu2017learning}, we disable and shorten two front legs of the ant such that it cannot walk without changing its gait to a large extent. Similar to Point-Maze-Shift, all imitaiton policies fail to maneuver the disabled ant to the right direction. As shown in Figure~\ref{fig:ant_film_strips}, reward functions learned by PEMIRL encourage the RL policy to orient the ant towards the demonstrated direction and move along that direction using two healthy legs, which is only possible when the inferred reward corresponds to the true underlying goal and is disentangled with the dynamics. In contrast, the learned reward of original AIRL as well as the discriminator output of Meta-InfoGAIL cannot infer the underlying goal of the task and provide precise supervision signal, which leads to the unsatisfactory performance of the induced RL policies. Quantitative results are presented in Table~\ref{tbl:reward_results}.

\section{Conclusion}
%
In this paper, we propose a new meta-inverse reinforcement learning algorithm, PEMIRL, which is able to efficiently infer robust reward functions that are disentangled from the dynamics and highly correlated with the ground-truth rewards under meta-learning settings. To our knowledge, PEMIRL is the first model-free Meta-IRL algorithm that can achieve this and scale to complex domains with continuous state-action spaces. PEMIRL generalizes to new tasks by performing inference over a latent context variable with a single demonstration, on which the recovered policy and reward function are conditioned. Extensive experimental results demonstrate the scalability and effectiveness of our method against strong baselines.
\section*{Acknowledgments}
This research was supported by Toyota Research Institute, NSF (\#1651565, \#1522054, \#1733686), ONR (N00014-19-1-2145), AFOSR (FA9550- 19-1-0024). The authors would like to thank Chris Cundy for discussions over the paper draft.
\bibliography{ref}

\begin{thebibliography}{42}
\providecommand{\natexlab}[1]{#1}
\providecommand{\url}[1]{\texttt{#1}}
\expandafter\ifx\csname urlstyle\endcsname\relax
  \providecommand{\doi}[1]{doi: #1}\else
  \providecommand{\doi}{doi: \begingroup \urlstyle{rm}\Url}\fi

\bibitem[Abbeel and Ng(2004)]{abbeel2004apprenticeship}
Pieter Abbeel and Andrew~Y Ng.
\newblock Apprenticeship learning via inverse reinforcement learning.
\newblock In \emph{Proceedings of the twenty-first international conference on
  Machine learning}, page~1, 2004.

\bibitem[Amodei et~al.(2016)Amodei, Olah, Steinhardt, Christiano, Schulman, and
  Man{\'{e}}]{amodei2016}
Dario Amodei, Chris Olah, Jacob Steinhardt, Paul~F. Christiano, John Schulman,
  and Dan Man{\'{e}}.
\newblock Concrete problems in {AI} safety.
\newblock \emph{arXiv preprint arXiv:1606.06565}, 2016.

\bibitem[Andrychowicz et~al.(2016)Andrychowicz, Denil, Colmenarejo, Hoffman,
  Pfau, Schaul, and de~Freitas]{Andrychowicz2016}
Marcin Andrychowicz, Misha Denil, Sergio~Gomez Colmenarejo, Matthew~W. Hoffman,
  David Pfau, Tom Schaul, and Nando de~Freitas.
\newblock Learning to learn by gradient descent by gradient descent.
\newblock \emph{arXiv preprint arXiv:1606.04474}, 2016.

\bibitem[Bengio et~al.(1991)Bengio, Bengio, and Cloutier]{bengio1991}
Yoshua Bengio, Samy Bengio, and Jocelyn Cloutier.
\newblock Learning a synaptic learning rule.
\newblock In \emph{IJCNN-91-Seattle International Joint Conference on Neural
  Networks}, 1991.

\bibitem[Duan et~al.(2016)Duan, Schulman, Chen, Bartlett, Sutskever, and
  Abbeel]{rl2}
Yan Duan, John Schulman, Xi~Chen, Peter~L. Bartlett, Ilya Sutskever, and Pieter
  Abbeel.
\newblock Rl{\textdollar}{\^{}}2{\textdollar}: Fast reinforcement learning via
  slow reinforcement learning.
\newblock \emph{arXiv preprint arXiv:1611.02779}, 2016.

\bibitem[Duan et~al.(2017)Duan, Andrychowicz, Stadie, Ho, Schneider, Sutskever,
  Abbeel, and Zaremba]{duan2017}
Yan Duan, Marcin Andrychowicz, Bradly Stadie, Jonathan Ho, Jonas Schneider,
  Ilya Sutskever, Pieter Abbeel, and Wojciech Zaremba.
\newblock One-shot imitation learning.
\newblock \emph{Neural Information Processing Systems (NIPS)}, 2017.

\bibitem[Finn et~al.(2016{\natexlab{a}})Finn, Christiano, Abbeel, and
  Levine]{finn2016connection}
Chelsea Finn, Paul Christiano, Pieter Abbeel, and Sergey Levine.
\newblock A connection between generative adversarial networks, inverse
  reinforcement learning, and energy-based models.
\newblock \emph{arXiv preprint arXiv:1611.03852}, 2016{\natexlab{a}}.

\bibitem[Finn et~al.(2016{\natexlab{b}})Finn, Levine, and
  Abbeel]{finn2016guided}
Chelsea Finn, Sergey Levine, and Pieter Abbeel.
\newblock Guided cost learning: Deep inverse optimal control via policy
  optimization.
\newblock In \emph{International Conference on Machine Learning}, pages 49--58,
  June 2016{\natexlab{b}}.

\bibitem[Finn et~al.(2017{\natexlab{a}})Finn, Abbeel, and Levine]{finn2017maml}
Chelsea Finn, Pieter Abbeel, and Sergey Levine.
\newblock Model-agnostic meta-learning for fast adaptation of deep networks.
\newblock In \emph{International Conference on Machine Learning},
  2017{\natexlab{a}}.

\bibitem[Finn et~al.(2017{\natexlab{b}})Finn, Yu, Zhang, Abbeel, and
  Levine]{finn2017one}
Chelsea Finn, Tianhe Yu, Tianhao Zhang, Pieter Abbeel, and Sergey Levine.
\newblock One-shot visual imitation learning via meta-learning.
\newblock 2017{\natexlab{b}}.

\bibitem[Fu et~al.(2017)Fu, Luo, and Levine]{fu2017learning}
Justin Fu, Katie Luo, and Sergey Levine.
\newblock Learning robust rewards with adversarial inverse reinforcement
  learning.
\newblock \emph{arXiv preprint arXiv:1710.11248}, 2017.

\bibitem[Gleave and Habryka(2018)]{multitaskirl}
Adam Gleave and Oliver Habryka.
\newblock Multi-task maximum entropy inverse reinforcement learning.
\newblock \emph{arXiv preprint arXiv:1805.08882}, 2018.

\bibitem[Goodfellow et~al.(2014)Goodfellow, Pouget-Abadie, Mirza, Xu,
  Warde-Farley, Ozair, Courville, and Bengio]{goodfellow2014generative}
Ian Goodfellow, Jean Pouget-Abadie, Mehdi Mirza, Bing Xu, David Warde-Farley,
  Sherjil Ozair, Aaron Courville, and Yoshua Bengio.
\newblock Generative adversarial nets.
\newblock In \emph{Advances in neural information processing systems}, pages
  2672--2680, 2014.

\bibitem[Hausman et~al.(2017)Hausman, Chebotar, Schaal, Sukhatme, and
  Lim]{hausman2017multi}
Karol Hausman, Yevgen Chebotar, Stefan Schaal, Gaurav Sukhatme, and Joseph~J
  Lim.
\newblock Multi-modal imitation learning from unstructured demonstrations using
  generative adversarial nets.
\newblock In \emph{Advances in Neural Information Processing Systems}, pages
  1235--1245, 2017.

\bibitem[Hausman et~al.(2018)Hausman, Springenberg, Wang, Heess, and
  Riedmiller]{hausman2018learning}
Karol Hausman, Jost~Tobias Springenberg, Ziyu Wang, Nicolas Heess, and Martin
  Riedmiller.
\newblock Learning an embedding space for transferable robot skills.
\newblock 2018.

\bibitem[Ho and Ermon(2016)]{gail}
Jonathan Ho and Stefano Ermon.
\newblock Generative adversarial imitation learning.
\newblock In \emph{Advances in Neural Information Processing Systems 29}, pages
  4565--4573. 2016.

\bibitem[Kingma and Welling(2013)]{vae}
Diederik~P Kingma and Max Welling.
\newblock Auto-encoding variational bayes.
\newblock \emph{arXiv preprint arXiv:1312.6114}, 2013.

\bibitem[Kuefler and Kochenderfer(2018)]{Kuefler2018}
Alex Kuefler and Mykel~J. Kochenderfer.
\newblock Burn-in demonstrations for multi-modal imitation learning.
\newblock In \emph{Proceedings of the 17th International Conference on
  Autonomous Agents and MultiAgent Systems}, AAMAS '18, 2018.

\bibitem[Levine(2018)]{levine2018reinforcement}
Sergey Levine.
\newblock Reinforcement learning and control as probabilistic inference:
  Tutorial and review.
\newblock \emph{arXiv preprint arXiv:1805.00909}, 2018.

\bibitem[Li and Malik(2017)]{li17optimize}
Ke~Li and Jitendra Malik.
\newblock Learning to optimize neural nets.
\newblock \emph{arXiv preprint arXiv:1703.00441}, 2017.

\bibitem[Li et~al.(2017)Li, Song, and Ermon]{li2017infogail}
Yunzhu Li, Jiaming Song, and Stefano Ermon.
\newblock Infogail: Interpretable imitation learning from visual
  demonstrations.
\newblock In \emph{Advances in Neural Information Processing Systems}, pages
  3812--3822, 2017.

\bibitem[Munkhdalai and Yu(2017)]{metanetworks}
Tsendsuren Munkhdalai and Hong Yu.
\newblock Meta networks.
\newblock \emph{International Conference on Machine Learning (ICML)}, 2017.

\bibitem[Ng and Russell(2000)]{ng2000irl}
Andrew~Y. Ng and Stuart~J. Russell.
\newblock Algorithms for inverse reinforcement learning.
\newblock In \emph{Proceedings of the Seventeenth International Conference on
  Machine Learning}, ICML '00, 2000.

\bibitem[Ng et~al.(1999)Ng, Harada, and Russell]{ng1999policy}
Andrew~Y Ng, Daishi Harada, and Stuart Russell.
\newblock Policy invariance under reward transformations: Theory and
  application to reward shaping.
\newblock In \emph{ICML}, volume~99, pages 278--287, 1999.

\bibitem[Peng et~al.(2018)Peng, Kanazawa, Toyer, Abbeel, and Levine]{peng2018}
Xue~Bin Peng, Angjoo Kanazawa, Sam Toyer, Pieter Abbeel, and Sergey Levine.
\newblock Variational discriminator bottleneck: Improving imitation learning,
  inverse rl, and gans by constraining information flow.
\newblock \emph{arXiv preprint arXiv:1810.00821}, 2018.

\bibitem[Rakelly et~al.(2019)Rakelly, Zhou, Quillen, Finn, and Levine]{pearl}
Kate Rakelly, Aurick Zhou, Deirdre Quillen, Chelsea Finn, and Sergey Levine.
\newblock Efficient off-policy meta-reinforcement learning via probabilistic
  context variables.
\newblock \emph{arXiv preprint arXiv:1903.08254}, 2019.

\bibitem[Ratliff et~al.(2006)Ratliff, Bagnell, and Zinkevich]{ratliff2006}
Nathan~D. Ratliff, J.~Andrew Bagnell, and Martin~A. Zinkevich.
\newblock Maximum margin planning.
\newblock In \emph{Proceedings of the 23rd International Conference on Machine
  Learning}, ICML '06, 2006.

\bibitem[Ross et~al.(2011)Ross, Gordon, and Bagnell]{ross2011reduction}
St{\'e}phane Ross, Geoffrey Gordon, and Drew Bagnell.
\newblock A reduction of imitation learning and structured prediction to
  no-regret online learning.
\newblock In \emph{Proceedings of the fourteenth international conference on
  artificial intelligence and statistics}, pages 627--635, 2011.

\bibitem[S{\ae}mundsson et~al.(2018)S{\ae}mundsson, Hofmann, and
  Deisenroth]{saemundsson2018meta}
Steind{\'o}r S{\ae}mundsson, Katja Hofmann, and Marc~Peter Deisenroth.
\newblock Meta reinforcement learning with latent variable gaussian processes.
\newblock \emph{arXiv preprint arXiv:1803.07551}, 2018.

\bibitem[Santoro et~al.(2016)Santoro, Bartunov, Botvinick, Wierstra, and
  Lillicrap]{mann}
Adam Santoro, Sergey Bartunov, Matthew Botvinick, Daan Wierstra, and Timothy
  Lillicrap.
\newblock Meta-learning with memory-augmented neural networks.
\newblock In \emph{International Conference on Machine Learning (ICML)}, 2016.

\bibitem[Schaal et~al.(2003)Schaal, Ijspeert, and Billard]{imitation_survey}
Stefan Schaal, Auke Ijspeert, and Aude Billard.
\newblock Computational approaches to motor learning by imitation.
\newblock \emph{Philosophical Transactions of the Royal Society of London B:
  Biological Sciences}, 2003.

\bibitem[Schmidhuber(1987)]{schmidhuber1987evolutionary}
J{\"u}rgen Schmidhuber.
\newblock \emph{Evolutionary principles in self-referential learning, or on
  learning how to learn: the meta-meta-... hook}.
\newblock PhD thesis, Technische Universit{\"a}t M{\"u}nchen, 1987.

\bibitem[Schulman et~al.(2015)Schulman, Levine, Moritz, Jordan, and
  Abbeel]{trpo}
John Schulman, Sergey Levine, Philipp Moritz, Michael~I. Jordan, and Pieter
  Abbeel.
\newblock Trust region policy optimization.
\newblock \emph{International Conference on Machine Learning}, 2015.

\bibitem[Sharma et~al.(2018)Sharma, Sharma, Rhinehart, and
  Kitani]{sharma2018directedinfogail}
Arjun Sharma, Mohit Sharma, Nicholas Rhinehart, and Kris~M. Kitani.
\newblock Directed-info {GAIL:} learning hierarchical policies from unsegmented
  demonstrations using directed information.
\newblock \emph{arXiv preprint arXiv:1810.01266}, 2018.

\bibitem[Todorov et~al.(2012)Todorov, Erez, and Tassa]{mujoco}
Emanuel Todorov, Tom Erez, and Yuval Tassa.
\newblock Mujoco: A physics engine for model-based control.
\newblock In \emph{International Conference on Intelligent Robots and Systems
  (IROS)}, 2012.

\bibitem[Xu et~al.(2018)Xu, Ratner, Dragan, Levine, and Finn]{xu2018learning}
Kelvin Xu, Ellis Ratner, Anca Dragan, Sergey Levine, and Chelsea Finn.
\newblock Learning a prior over intent via meta-inverse reinforcement learning.
\newblock \emph{arXiv preprint arXiv:1805.12573}, 2018.

\bibitem[Yu et~al.(2018{\natexlab{a}})Yu, Abbeel, Levine, and Finn]{yu2018hil}
Tianhe Yu, Pieter Abbeel, Sergey Levine, and Chelsea Finn.
\newblock One-shot hierarchical imitation learning of compound visuomotor
  tasks.
\newblock \emph{arXiv preprint arXiv:1810.11043}, 2018{\natexlab{a}}.

\bibitem[Yu et~al.(2018{\natexlab{b}})Yu, Finn, Xie, Dasari, Zhang, Abbeel, and
  Levine]{yu2018daml}
Tianhe Yu, Chelsea Finn, Annie Xie, Sudeep Dasari, Tianhao Zhang, Pieter
  Abbeel, and Sergey Levine.
\newblock One-shot imitation from observing humans via domain-adaptive
  meta-learning.
\newblock \emph{Robotics: Science and Systems (R:SS)}, 2018{\natexlab{b}}.

\bibitem[Zhang et~al.(2017)Zhang, McCarthy, Jow, Lee, Goldberg, and
  Abbeel]{vr_imitation}
Tianhao Zhang, Zoe McCarthy, Owen Jow, Dennis Lee, Ken Goldberg, and Pieter
  Abbeel.
\newblock Deep imitation learning for complex manipulation tasks from virtual
  reality teleoperation.
\newblock \emph{arXiv preprint arXiv:1710.04615}, 2017.

\bibitem[Zhao et~al.(2018)Zhao, Song, and Ermon]{zhao2018information}
Shengjia Zhao, Jiaming Song, and Stefano Ermon.
\newblock The information autoencoding family: A lagrangian perspective on
  latent variable generative models.
\newblock \emph{arXiv preprint arXiv:1806.06514}, 2018.

\bibitem[Ziebart(2010)]{ziebart2010modeling}
Brian~D Ziebart.
\newblock Modeling purposeful adaptive behavior with the principle of maximum
  causal entropy.
\newblock 2010.

\bibitem[Ziebart et~al.(2008)Ziebart, Maas, Bagnell, and
  Dey]{ziebart2008maximum}
Brian~D Ziebart, Andrew~L Maas, J~Andrew Bagnell, and Anind~K Dey.
\newblock Maximum entropy inverse reinforcement learning.
\newblock In \emph{Aaai}, volume~8, pages 1433--1438. Chicago, IL, USA, 2008.

\end{thebibliography}
\bibliographystyle{plainnat}
\newpage
\appendix
\section{Proof of Lemma~\ref{lemma:airl}}\label{appendix:airl}
From Section 2.3 and 2.4 in \citep{levine2018reinforcement}, we know that the policy whose induced trajectory distribution is Equation~(\ref{eq:conditional}) takes the following energy-based form:
\begin{align*}
    \pi_\theta(a_t|s_t, m) &= \exp(Q_{\text{soft}}(s_t, a_t, m) - V_{\text{soft}}(s_t, m))\\
    Q_{\text{soft}}(s_t, a_t, m) &= f_\theta(s_t, a_t, m) + \log \mathbb{E}_{s_{t+1} \sim P(\cdot|s_t, a_t, m)}[\exp(V_{\text{soft}}(s_{t+1}, m))]\\
    V_{\text{soft}}(s_t, m) &= \log \int_{\gA}  \exp(Q_{\text{soft}}(s_t, a', m)) d a'
\end{align*}
which corresponds to the optimal policy to the following entropy regularized reinforcement learning problem (for a certain value of $m$):
\begin{align}
    \max_\pi \mathbb{E}_\pi \left[ \sum_{t=1}^T f_\theta(s_t, a_t, m) - \log \pi(a_t|s_t, m) \right]\label{eq:entropy-rl}
\end{align}
From Section~\ref{sec:preliminary}, we know that Equation~(\ref{eq:entropy-rl}) is exactly the training objective for the adaptive sampler $\pi_\omega$ in AIRL. Thus, the trajectory distribution of the optimal policy $\pi_{\omega^*}$ matches $p_\theta(\tau|m)$ defined in Equation~(\ref{eq:conditional}).

\section{Proof of Lemma~\ref{lemma:gradient}}\label{appendix:gradient}
First, the gradient of $\gL_\text{info}(\theta, \psi)$ w.r.t. $\theta$ can be written as:
\begin{align}
    \frac{\partial}{\partial \theta} \gL_\text{info}(\theta, \psi) = \mathbb{E}_{m \sim p(m), \tau \sim p(\tau|m, \theta)} \log q(m|\tau, \psi) \frac{\partial}{\partial \theta} \log p_\theta(\tau|m)\label{eq:gradient-theta}
\end{align}

As $p_\theta(\tau|m)$ is an energy-based distribution (Equation~(\ref{eq:conditional})), we need to derive the gradient of $\log p(\tau|m, \theta)$ w.r.t. $\theta$:
\begin{align}
    \frac{\partial}{\partial \theta} \log p(\tau|m, \theta) &= \frac{\partial}{\partial \theta} \left[\log\left(\eta(s_1) \prod_{t=1}^T P(s_{t+1}|s_t,a_t)\right) + \sum_{t=1}^T f_\theta(s_t,a_t,m) - \log Z(\theta)\right]\\
    &= \sum_{t=1}^T \frac{\partial}{\partial \theta} f_\theta(s_t,a_t,m) - \frac{\partial}{\partial \theta} \log Z(\theta)\\
    &= \sum_{t=1}^T \frac{\partial}{\partial \theta} f_\theta(s_t,a_t,m) - \mathbb{E}_{\tau \sim p(\tau|m, \theta)}\left[\sum_{t=1}^T \frac{\partial}{\partial \theta} f_\theta(s_t,a_t,m)\right] \label{eq:derive-theta}
\end{align}
Substituting Equation~(\ref{eq:derive-theta}) into Equation~(\ref{eq:gradient-theta}), we get:
\begin{align}
    \mathbb{E}_{m \sim p(m), \tau \sim p_{\theta}(\tau|m)}\left[ \log q_\psi(m|\tau) \left[ \sum_{t=1}^T \frac{\partial}{\partial \theta} f_\theta(s_t,a_t,m) - \mathbb{E}_{\tau' \sim p_{\theta}(\tau|m)}\sum_{t=1}^T \frac{\partial}{\partial \theta} f_\theta(s'_t,a'_t,m)\right] \right]\nonumber
\end{align}
With Lemma~\ref{lemma:airl}, we know that when $\omega$ is trained to optimality, we can sample from $p_{
\pi_\omega^*}(\tau|m)$ to construct an unbiased gradient estimation.

\section{Meta-Testing Procedure of PEMIRL}\label{appendix:meta-test}
We summarize the meta-test stage of PEMIRL for adapting reward functions to new tasks in Algorithm~\ref{meta_test_algorithm}.
\begin{algorithm}[h]
   \caption{PEMIRL Meta-Test for Reward Adaptation}
   \label{meta_test_algorithm}
\begin{algorithmic}
   \STATE {\bfseries Input:} A test context variable $m \sim p(m)$, a test expert demonstration $\tau_E \sim p_{\pi_E}(\tau|m)$,
    and ground-truth reward $r(s,a,m)$.
   \STATE Infer the latent context variable from the test demonstration: $\hat{m} \sim q_\psi(m|\tau_E)$.
   \STATE Train a policy using TRPO w.r.t. adapted reward function $f_\theta(s, a, \hat{m})$.
   \STATE Evaluate the learned policy with $r(s,a,m)$.
\end{algorithmic}
\end{algorithm}

\section{Graphical Model of PEMIRL}
\label{app:graphical_model}
\begin{figure}[t]
    \centering
    \vspace{-0.3cm}
    \includegraphics[width=0.7\columnwidth]{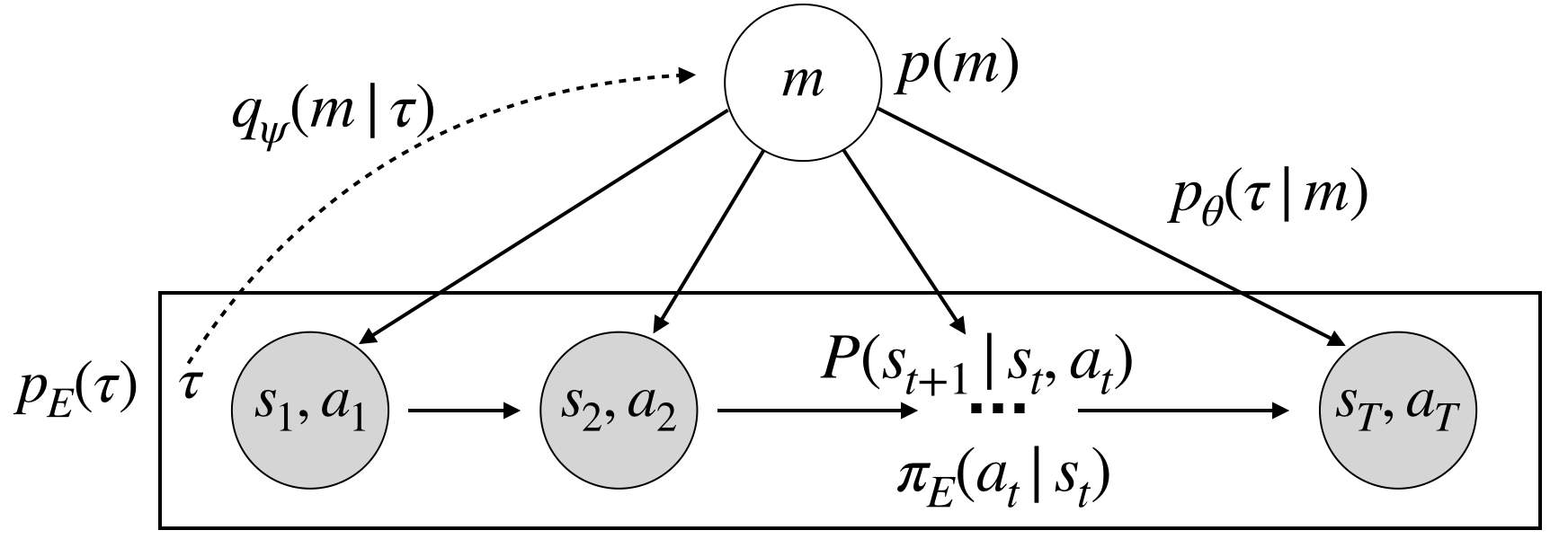}
    \vspace{-0.3cm}
    \caption{Graphical model underlying PEMIRL.
    }
    \vspace{-0.4cm}
    \label{fig:graphical_model}
\end{figure}

Here we show the graphical model of the PEMIRL framework in Figure~\ref{fig:graphical_model}.

\section{Ablation Studies}
\label{app:ablation}

In this section, we perform ablation studies on the sensitivity of the latent dimensions, importance of the mutual information loss ($\mathcal{L}_\text{info}$) term, and stochasticity of the environment. We conduct each ablation study on the Point-Maze-Shift environment to evaluation the reward adaptation performance.

\begin{table}[h]
\setlength\aboverulesep{4.5pt}\setlength\belowrulesep{4.5pt}
\begin{small}
	\begin{minipage}{0.25\linewidth}
	\centering
	\vspace{-3pt}
    \begin{tabular}{c|c}
    \toprule
     latent dim. &  return \\
     \midrule
      1 &   $-10.58 \pm 1.27$  \\ 
      3  & $-14.13 \pm 1.21$ \\
      5 &  $-15.41 \pm 1.40$ \\
      \bottomrule
    \end{tabular}
    \caption{PEMIRL is robust to latent dimensions.}
    \label{tbl:latent}
	\end{minipage}
	\hfill
	\begin{minipage}{0.3\linewidth}
	\centering
	\vspace{1pt}
    \begin{tabular}{c|c}
    \toprule
     method &  return \\
     \midrule
      PEMIRL w/o MI &   $-39.24 \pm 3.48$  \\ 
      PEMIRL  & $-14.13 \pm 1.21$ \\
      \bottomrule
    \end{tabular}
    \caption{The MI term is important for training PEMIRL.}
    \label{tbl:mi}
	\end{minipage}
	\hfill
	\begin{minipage}{0.25\linewidth}
	\centering
	\vspace{0pt}
    \begin{tabular}{c|c}
    \toprule
     method &  return \\
     \midrule
      Meta-IL &   $-30.58 \pm 4.17$  \\ 
      PEMIRL & $-17.39 \pm 0.84$\\
      \bottomrule
    \end{tabular}
    \caption{PEMIRL excels in stochastic env.}
    \label{tbl:schochastic}
	\end{minipage}
	\hfill
	\end{small}
\end{table}

\prg{Sensitivity of the latent dimension.} We first investigate the sensitivity of different latent dimensions by running PEMIRL with latent dimension picked from $\{1, 3, 5\}$ on Point-Maze-Shift where the ground-truth latent dimension is $3$. The results are summarized in Table~\ref{tbl:latent}. We can observe that PEMIRL with various latent dimension specifications all outperform the best baseline (return -28.61) stably and is hence robust to dimension mis-specifications.

\prg{Importance of $\mathcal{L}_\text{info}$.} As shown in Table~\ref{tbl:mi}, the reward function learned by PEMIRL without the mutual information objective failed to induce a good policy in the reward adaptation setting, which demonstrates the importance of using $\mathcal{L}_\text{info}$.

\prg{Performance on stochastic environment.} We create a stochastic version of Point-Maze-Shift (maze size: $60 \times 100$ cm) by changing its deterministic transition dynamics into a stochastic one. Specifically, $p(s_{t+1}|s_t,a_t)$ is now realized as a Gaussian with standard deviation being 1 cm. As shown in Table~\ref{tbl:schochastic}, the average return of PEMIRL outperforms the best baseline Meta-IL by a large margin.

\section{Additional Experimental Details}\label{appendix:experiment}

\subsection{Network Architectures}

For all methods except AIRL, $q_\psi(m|\tau)$ and $\pi_\omega(a|s,m)$ are represented as $2$-layer fully-connected neural networks with $128$ and $64$ hidden units respectively and ReLU as the activation function. 

Following \cite{fu2017learning}, to alleviate the reward ambiguity problem, we represent the reward function with two components (a context-dependent disentangled reward estimator $r_{\theta}(s,m)$ and a context-dependent potential function $h_{\phi}(s,m)$):
$$f_{\theta, \phi}(s_t,a_t,s_{t+1},m) = r_{\theta}(s_t,m) + \gamma h_{\phi}(s_{t+1},m) - h_{\phi}(s_t,m)$$
Here $r_\theta(s,m)$ and $h_\phi(s,m)$ are realized as a $2$-layer fully-connected neural networks with $32$ hidden units.

\subsection{Environment Details}
\prg{Point-Maze}. The ground-truth reward corresponds to negative distance toward the goal position as well as controlling the pointmass from moving too fast. We use $100$ meta-training tasks and $30$ meta-training tasks.

\prg{Ant}. The ground-truth reward corresponds to moving as far as possible forward or backward without being flipped. We have $2$ tasks in this domain.

\prg{Sweeper}. The ground-truth reward is the negative distance from the sweeper to the object plus the negative distance from the object to the goal position. We train all methods on $100$ meta-training tasks and test them on $30$ meta-test tasks.

\prg{Sawyer Pushing}. The ground-truth reward in this domain is similar to Sweeper, and we also use $100$ meta-training tasks and $30$ meta-test tasks.

\subsection{Training Details}
\textit{Training the policy.} During training TRPO, we use an entropy regularizer $1.0$ for Point-Maze, and $0.1$ for the other three domains. We find that adding an imitation objective in PEMIRL that maximizes the log-likelihood of the sampled expert trajectory conditioned on the latent context variable inferred by $q_\psi$ with scaling factor $0.01$ accelerates policy training.

\textit{Training the inference network and the reward model.} We train $q_\psi(m|\tau)$, $r_\theta(s,m)$ and $h_\phi(s,m)$ using the Adam optimizer with default hyperparameters.

\textit{Scaling up the mutual information regularization.} Note that in Equation~\ref{eq:objective}, $\beta$ does not necessarily need to be equal to $1$. Adjusting $\beta$ is equivalent to scaling $\mathcal{L}_\text{info}(\theta, \psi)$. We scale $\mathcal{L}_\text{info}(\theta, \psi)$ by $0.1$ for all of our experiments.

\textit{Policy and inference network initialization.} We initialize and $q_\psi(m|\tau)$ using Meta-IL discussed in Section~\ref{sec:experiments} while randomly initializing the policy $\pi_\omega(a|s,m)$.

\textit{Stabilizing adversarial training.} As in ~\cite{fu2017learning}, we mix policy samples generated from previous $20$ training iterations and use them as negatives when training the disriminator. We find that such a strategy prevents the discriminator from overfitting to samples from the current iteration.

\subsection{Data Efficiency}

During meta-training, for the Point-Maze environment, it takes about 32M simulation steps to converge (similar to other methods such as Meta-InfoGAIL that takes 28M), which amounts to about 2 hours on one Nvidia Titan-Xp GPU; for the Ant environment, it takes about 13.8M simulation steps (Meta-InfoGAIL takes 12M) and about 40 hours on the same hardware (the state-action dimension is much larger than that of Point-Maze).

At meta-testing phase, the data efficiency of PEMIRL is comparable to RL training with the oracle ground-truth reward as shown in Table~\ref{tbl:data_efficiency}.

\begin{table}[h]
    \begin{center}
    \setlength\aboverulesep{4.5pt}\setlength\belowrulesep{4.5pt}
    \begin{small}
    \begin{tabular}{l|c|c}
    \toprule
        & Point-Maze-Shift & Disabled-Ant\\
      \midrule
      RL w/ oracle reward & 4M env steps & 15M env steps\\
      PEMIRL & 5.4M env steps & 18M env steps\\
      \bottomrule
    \end{tabular}
    \end{small}
    \end{center}
    \vspace{0.2cm}
    \caption{Comparison on data efficiency between RL trained with reward learned by PERMIL and RL trained with oracle reward. The methods have been shown to have similar data efficiency on Point-Maze-Shift and Disabled-Ant.}
    \vspace{-0.4cm}
    \label{tbl:data_efficiency}
\end{table}

\end{document}